\documentclass[11pt,a4paper]{article}
\usepackage{geometry}                % See geometry.pdf to learn the layout options. There are lots.
\usepackage{graphicx}
\usepackage{amssymb}
\usepackage{amsmath}
\usepackage{amsthm}
\usepackage{bbold}
\theoremstyle{plain}
\newtheorem{theorem}{Theorem}

\newtheorem{example}[theorem]{Example}

\newcommand{\bh}{\boldsymbol{h}}

\newcommand{\setb}{{\{0,1\}}}
\newcommand{\vc}[1]{{d({#1})}}

\renewcommand{\restriction}{\mathord{|}}
\newcommand{\pr}[2]{{#1}\restriction_{{#2}}}
\newcommand{\cH}{\mathcal{H}}
\newcommand{\bine}{H} % Use capital H so as not to confuse with little h used to denote hypothesis

\title{The VC-Dimension of Similarity Hypothesis Spaces}
\author{{\bf Mark Herbster,~~~Paul Rubenstein,~~~James Townsend} \vspace{.08in}\\ 
Department of Computer Science \\ University College London \\
Gower Street, London WC1E 6BT, England, UK \\
 \em{\small \{m.herbster,~paul.rubenstein.14,~james.townsend.14\}@ucl.ac.uk} \\
}
\begin{document}
\maketitle
\begin{abstract}
\iffalse
%Given a hypothesis space $\cH \subseteq \setb^X$ one may induce a {\em similarity} hypothesis space $\cH^{(s)}\subseteq \setb^{X\times X}$, 
%where $\bh^{(s)}(x',x'') \in \cH^{(s)}$ iff there exists an $\bh\in\cH$ such that $|\bh(x')-\bh(x'')| = 1- \bh^{(s)}(x',x'')$ for all $(x',x'')\in X\times X$.  In other words, given that each hypothesis function $\bh\in\cH$ labels each pattern $x\in X$ as a 0 or 1, each such hypothesis  function % $\bh\in\cH$ 
%induces a similarity hypothesis $\bh^{(s)}\in \cH^{(s)}$ function on {\em pairs} of patterns $(x',x'')\in X\times X$ which labels a pair 
% as 1 ({\sc similar}) iff  the same pair of patterns share the same label with respect to $\bh$ and 0 ({\sc disimilar}) otherwise.  We show in this note that $\mbox{{\sc vc-dimension}}(\cH^{(s)}) \in \Theta(\mbox{{\sc vc-dimension}}(\cH))$.
\fi
Given a set $X$ and a function $h:X\longrightarrow\{0,1\}$ which labels each element of $X$ with either $0$ or $1$, we may define a function $h^{(s)}$ to measure the \emph{similarity} of pairs of points in $X$ according to $h$. Specifically, for $h\in \{0,1\}^X$ we define $h^{(s)}\in \{0,1\}^{X\times X}$ by $h^{(s)}(w,x):= \mathbb{1}[h(w) = h(x)]$.
This idea can be extended to a set of functions, or \emph{hypothesis space} $\mathcal{H} \subseteq \{0,1\}^X$ by defining a \emph{similarity hypothesis space} $\mathcal{H}^{(s)}:=\{h^{(s)}:h\in\mathcal{H}\}$.  We show that   $\mbox{{\sc vc-dimension}}(\cH^{(s)}) \in \Theta(\mbox{{\sc vc-dimension}}(\cH))$.
\end{abstract} 

\section{Introduction} 
Consider the problem of learning from examples.  We may learn by receiving {\em class} labels as feedback: `this is a {\bf dog}', `that is a {\bf wolf}' , `there is a {\bf cat}', etc.   We may also learn by receiving {\em similarity} labels: `these are the {\bf same}', `those are {\bf different}' and so forth.  In this note we study the problem of learning with similarity versus class labels.  Our approach is to use the {\em VC-dimension}~\cite{vc-ucrfep-71} to study the fundamental difficulty of this learning task.

In the supervised learning model we are given a training set of patterns and associated labels.  
The goal is then to find a hypothesis function that maps patterns to labels that will predict with few errors  on future data ({\em small generalization error}).   
A classic approach to this problem is empirical risk minimisation. Here the procedure is to choose a hypothesis 
from a set of hypothesis functions ({\em hypothesis space}) that `fits' the data as closely as possible.  
If the hypothesis is from a hypothesis space with small VC-dimension %(see~\eqref{eq:vcdef}) 
and fits the data well then we are likely to predict well on future data~\cite{vc-ucrfep-71,behw-lvd-89}.   
The number of examples required to have small generalisation error with high probability is called the {\em sample complexity}.   In the {\em uniform learnability} model the VC-dimension gives a nearly matching upper and 
lower bound on the sample complexity~\cite{behw-lvd-89,ehkv-glbnenl-89}.
In Theorem~\ref{thm:vcdim} we demonstrate that the VC-dimension of a hypothesis space with respect to similarity-labels is proportionally bounded by the VC-dimension with respect to class-labels indicating that the sample complexities within the two feedback  settings are comparable. That is, the fundamental difficulties of the two learning tasks are comparable.
\subsection*{Related work} We are motivated by the results of~\cite{ghp-ospndkug-13}.  Here the authors considered the problem of similarity prediction in the online mistake bound model~\cite{litt88}.  In~\cite[Theorem~1]{ghp-ospndkug-13}  it was found that  given a basic algorithm for class-label prediction with a mistake bound there exists an algorithm for similarity-label prediction with a mistake bound which was larger by no more than a constant factor.  In this work we find an analogous result in  terms of the VC-dimension.  
% Other work of interest ...

\section{The VC-dimension of similarity hypothesis spaces}
A hypothesis space $\cH \subseteq \{0,1\}^X$ is a set of functions from some set of patterns $X$ to the set of labels $Y=\{0,1\}$ in the two-class setting.  The \emph{restriction} of a function $h \in\{0,1\}^X$ to a subset $X'\subseteq X$ is the function $\pr{h}{X'} \in \{0,1\}^{X'}$ with $\pr{h}{X'}(x):=h(x)$ for each $x \in X'$. Analogously, one can define the restriction of a hypothesis space as $\pr{\cH}{X'} := \{\pr{h}{X'} : h\in\cH\}$.

A subset $X' \subseteq X$ is said to be \emph{shattered} by $\mathcal{H}$ if $\pr{\cH}{X'}=\{0,1\}^{X'}$, that is if the restriction contains \emph{all possible} functions from $X'$ to $\{0,1\}$. The VC-dimension~\cite{vc-ucrfep-71} of a hypothesis space $\cH \subseteq \{0,1\}^X$, denoted $d(\cH)$, is the size of the largest subset of $X$ which is shattered by $\cH$, that is
\begin{equation*}\label{eq:vcdef}
\vc{\cH} := \max_{X' \subseteq X} \{ |X'| : \pr{\cH}{X'}=\{0,1\}^{X'} \}\,.
\end{equation*}
\emph{Sauer's lemma}~\cite{vc-ucrfep-71,Sauer72,shelah1972}, which gives a lower bound for the VC-dimension of a hypothesis space, will be used for proving our main result. It states that for a hypothesis space $\mathcal{H}\subseteq \{0,1\}^X$, if
\begin{equation}\label{eq:sauer}
|\mathcal{H}|>\sum_{k=0}^{m-1} \binom{|X|}{k}
\end{equation}
then $\vc{\mathcal{H}}\geq m$.
%A {\em similarity} hypothesis space $\cH^{(s)}$ is induced from the hypothesis space $\cH$ by mapping each hypothesis $\bh \in \cH$ to a similarity hypothesis $\bh^{(s)}: = s(\bh)$ ($\cH^{(s)} := s(\cH)$) by a mapping $s : \{0,1\}^X \rightarrow \{0,1\}^{X\times X}$.  
%We define, 
%\[ 
%s(\bh) := ([h_{x'}=h_{x''}])_{(x',x'')\in X \times X}\,,
%\]
%where we use the Iverson bracket notation $\id{\mbox{\sc predicate}}=1$ if the predicate is true and $\id{\mbox{\sc predicate}}=0$ if false.

Given a function $h:X\longrightarrow\{0,1\}$, we may define a function $h^{(s)}$ to measure the \emph{similarity} of pairs of points in $X$ according to $h$. Specifically, for $h\in \{0,1\}^X$ we define $h^{(s)}\in \{0,1\}^{X\times X}$ by $h^{(s)}(w,x):= \mathbb{1}[h(w) = h(x)]$, where $\mathbb{1}$ is the indicator function.
This idea can be extended to a hypothesis space $\mathcal{H}$ by defining the \emph{similarity hypothesis space} $\mathcal{H}^{(s)}:=\{h^{(s)}:h\in\mathcal{H}\}$.  
%For example, given a hypothesis
%\[
%\bh := (1,0,0) \text{ then } s(\bh) := \left( \begin{array}{ccc}
%1 & 0 & 0 \\
%0 & 1 & 1 \\
%0 & 1 & 1 \end{array} \right)\, .
%\]  
%Observe that $|\cH^{(s)}| \le |\cH|$ as for example $s((0,0,\ldots,0)) = s((1,1,\ldots,1))$ and more generally $s((h_x)_{x\in X}) = s((1-h_x)_{x\in X})$.
We now give our central result,
\begin{theorem}\label{thm:vcdim}
Given a hypothesis space $\cH \subseteq \setb^X$,
\begin{equation*}
\vc{\cH} -1  \le \vc{\cH^{(s)} } \le \delta \vc{\cH}\,,
\end{equation*}
with $\delta = 4.55$.
\end{theorem}
\begin{proof}
%If $m=\vc{\cH}$ then there exists a set $T= \{x_1,\ldots,x_m\} \subseteq X$ with $\pr{\cH}{T} = 2^m$,
%now construct,
%\[ S = 
%\{(x_1,x_2),(x_2,x_3),\ldots,(x_{m-1},x_m)\}\,,
%\] observe that $|\pr{s(\cH)}{S}| = 2^{m-1}$. Thus the left-hand inequality is shown.
For the left hand inequality, let $n := d(\mathcal{H})$ and pick a set $T = \{x_1, x_2, \ldots , x_n\}$ of size $n$ which is shattered by $\mathcal{H}$. Then let $T' = \{(x_1, x_2), (x_2, x_3), \ldots , (x_{n-1}, x_{n})\}$. To demonstrate that $T'$ is shattered by $\mathcal{H}^{(s)}$, let $g \in \{0,1\}^{T'}$ be any mapping from $T'$ to $\{0,1\}$. Then since $T$ is shattered by $\mathcal{H}$ we may find a map $h \in \mathcal{H}$ with $h(x_1) = 0$ and
\begin{equation*}
h(x_{i+1}) = \left\{
	\begin{array}{l l}
		h(x_i) & \quad \text{if } g(x_i, x_{i+1}) = 1\\
		1 - h(x_i) & \quad \text{if } g(x_i, x_{i+1}) = 0
	\end{array}
\right.
\end{equation*}
for $i = 1, \ldots, n-1$. Observe that $g = \pr{h^{(s)}}{T'}$. Since $g$ was chosen arbitrarily, we may conclude that $T'$ is indeed shattered by $\mathcal{H}^{(s)}$, and therefore $d(\mathcal{H}^{(s)}) \geq |T'| = d(\mathcal{H}) - 1$.

For the right hand inequality, first let $M := \vc{\mathcal{H}^{(s)}}$ and then pick a set $U = \{(w_1, x_1), (w_2, x_2), \ldots , (w_M, x_M)\}$ of size $M$ in $X\times X$ which is shattered by $\mathcal{H}^{(s)}$. Let $V = \{w_1, w_2, \ldots, w_M, x_1, x_2, \ldots, x_M\}$ and note that $|\pr{\mathcal{H}}{V}| \geq |\pr{\mathcal{H}^{(s)}}{U}| = 2^M$. This is because any two maps $h$ and $g$ which agree on $V$ will induce maps $h^{(s)}$ and $g^{(s)}$ which agree on $U$, so $\pr{\mathcal{H}^{(s)}}{U}$ cannot possibly contain more maps than $\pr{\mathcal{H}}{V}$.
\iffalse
%then by the definition of VC-dimension there exists a set of pairs of indices $S = \{(i_1,j_1), (i_2,j_2), \ldots, (i_M,j_M)\}$ such that the projection of $s(\cH)$ onto the indices in $S$ contains all possible $2^M$ configurations.

%Let $T:=\{i_1, i_2, \ldots, i_M, j_1, j_2, \ldots, j_M\}$ and note that $|\pr{\cH}{T}| \ge |\pr{s(\cH)}{S}|={2^M}$, as $\pr{s(\pr{\cH}{T})}{S} = \pr{s(\cH)}{S}$ for if this were not the case, then there would have to be at least one pair of vectors $\ba$ and $\bb$ in $\cH$ such that $\pr{\ba}{T} = \pr{\bb}{T}$ and $\pr{s(\ba)}{S} \neq \pr{s(\bb)}{S}$, which is not possible.
%Now we recall Sauer's Lemma~\cite{vc-ucrfep-71,Sauer72,shelah1972}, which states that that for $\cH' \subseteq \{0,1\}^X$, if
%\begin{equation}\label{eq:sauer}
%|\cH'| > \sum_{k=0}^{m-1} \binom{|X|}{k}.
%\end{equation}
%then $\vc{\mathcal{H}'} \geq m$.
\fi
Using this fact, and applying Sauer's Lemma (see~\eqref{eq:sauer}) to $\pr{\mathcal{H}}{V}$, we see that if
\begin{equation*}
2^M > \sum_{k = 0}^{m-1}\binom{|V|}{k}
\end{equation*}
then $\vc{\mathcal{H}} \geq \vc{\pr{\mathcal{H}}{V}} \geq m$.

%Applying the lemma to $\cH' := \pr{\cH}{V}$, if
%\begin{equation*}
%\sum_{k=0}^{m-1} \binom{|T|}{k} < 2^M \mbox{ then } \vc{\cH} \geq m\,,
%\end{equation*}
%as $\vc{\pr{\cH}{T}} \geq m$ implies $\vc{\cH} \geq m$.
\iffalse
Applying Sauer's lemma~\cite{vc-ucrfep-71,Sauer72,shelah1972} to $\pr{\cH}{T}$ therefore implies that if
\begin{equation*}
\sum_{k=0}^{m-1} \binom{|T|}{k} < 2^M
\end{equation*}
then $\vc{\cH} \geq m$.
\fi

%The binary entropy is $\bine(\epsilon) := \epsilon \log_2 \frac{1}{\epsilon} + (1-\epsilon) \log_2 \frac{1}{1-\epsilon}$,  for which  we use the following inequality on sums of binomials coefficients (see e.g.,~\cite[Lemma~16.19]{FlumGrohe06}),
Now note the following inequality (see e.g.,~\cite[Lemma~16.19]{FlumGrohe06}), which bounds a sum of binomial coefficients:
\begin{equation}\label{eq:entineq}
\sum_{i=0}^{\lfloor{\epsilon n}\rfloor} \binom{n}{i} \leq 2^{\bine(\epsilon)n} \quad (0 < \epsilon < 1/2)\,,
\end{equation}
where $\bine(\epsilon) := \epsilon \log_2 \frac{1}{\epsilon} + (1-\epsilon) \log_2 \frac{1}{1-\epsilon}$ denotes the \emph{binary entropy} function. If we set
$m = 1 + \lfloor2\epsilon M\rfloor$ for some $\epsilon<\frac{1}{2}$ such that $\bine(\epsilon)<\frac{1}{2}$, we have
\begin{equation*}
\sum_{k=0}^{m-1} \binom{|V|}{k} = \sum_{k=0}^{\lfloor2\epsilon M\rfloor} \binom{|V|}{k} \leq \sum_{k=0}^{\lfloor2\epsilon M\rfloor} \binom{2M}{k} \leq 2^{2M\bine(\epsilon)} < 2^M
\end{equation*}
using~\eqref{eq:entineq} and that $|V| \leq 2M$ from the definition of $V$.
Thus Sauer's lemma can be applied with the above value of $m$ and hence
\[ \vc{\cH} \geq 1 + \lfloor2\epsilon M\rfloor \geq 2\epsilon M = 2\epsilon \vc{\cH^{(s)}}\,,\] 
as long as $\bine(\epsilon) < 1/2$. Observe that $\epsilon = .11$ satisfies this condition and thus we have that
\begin{equation*}
 \vc{\cH^{(s)}} \leq 4.55  \vc{\cH}\,.
\end{equation*}
\end{proof}
\section{Discussion}
In the following, we give a family of examples where the VC-dimension of the similarity hypothesis space is exactly twice that of the original space.   We use the following notation for the set of the first $n$ natural numbers $[n] := \{1,2,\ldots,n\}$.
\begin{example}\label{ex:sparsek}
For the hypothesis space of \emph{k-sparse vectors}, $\cH_{k} := \{h \in \{0,1\}^{[n]} : \sum_{i=1}^n h(i) \leq k\}$,  \[ \vc{\cH_{k}} = k  \text{ and } \vc{\cH_k^{(s)}}=2k\,,\]
provided that $n \geq 2k+1$.
\end{example}

\begin{proof}
Let $X := [n]$. 
Firstly note that $\vc{\cH_{k}}\geq k$, since any subset $T\subseteq X$ with $|T| \leq k$ is shattered by $\cH_{k}$. If $T'\subseteq X$ with $|T'|>k$ then $T'$ cannot possibly be shattered by $\cH_k$ since there is no element in $\cH_{k}$ that labels all elements of $T'$ as 1. Therefore $\vc{\cH_{k}} = k$.

%since for any subset $T'\subseteq T$ there is some element $h \in \cH_{k}$ with $h(x)=1 \forall x \in T' and h(x) = 0 \forall x \in X\setminus T'$.

To see that $\vc{\cH_k^{(s)}} \geq 2k$, let $U = \{(x_1,x_2),(x_2,x_3),\ldots ,(x_{2k},x_{2k+1})\}$ for any distinct elements $x_1, x_2, \ldots, x_{2k+1}\in X$ and note that $|U| = 2k$. To show that $U$ is shattered by $\cH_k^{(s)}$, let $g \in \{0,1\}^U$ be any function from $U$ to $\{0,1\}$. We need to find an $h \in \cH_k$ such that $g = \pr{h^{(s)}}{U}$. Two functions in $\{0,1\}^X$ which satisfy the condition $g = \pr{h^{(s)}}{U}$ are $h_0$ and $h_1$ defined by $h_0(x_1) = 0$, $h_1(x_1) = 1$ and
\begin{equation*}
\begin{array}{ll}
h_j(x_{i+1}) & = \left\{ 
	\begin{array}{l l}
		h_j(x_i) & \quad \text{if } g(x_i, x_{i+1}) = 1\\
		1 - h_j(x_i) & \quad \text{if } g(x_i, x_{i+1}) = 0
	\end{array}
\right. \\
h_j(x) & = \begin{array}{l l}
		  \enspace\,\, 0 & \quad\quad\quad\quad\enspace\,\, \forall x \not\in  \{x_1,x_2,\ldots,x_{2k+1}\} %#sorrynotsorry
	\end{array}
\end{array}
\end{equation*}
for $i = 1,\ldots,2k$ and $j = 0,1$. Observe that by construction, $h_0(x_i) + h_1(x_i) = 1 $ for each $i = 1,\ldots,2k+1$ and therefore $\sum_{i = 1}^{2k+1}h_0(x_i) + \sum_{i = 1}^{2k+1}h_1(x_i) = \sum_{i = 1}^{2k+1}[ h_0(x_i) + h_1(x_i)] = 2k+1$. This means that we must have $\sum_{i = 1}^{2k+1} h_j(x_i) \leq k$ for some $j$ and hence $h_j \in \cH_{k}$ with $\pr{h_j^{(s)}}{U} = g$. This proves that $\vc{\cH_{k}^{(s)}} \geq 2k$.

Now suppose, for a contradiction, that $\vc{\cH_{k}^{(s)}} > 2k$. Then there is some set \\$E = \{(u_1,v_1),(u_2,v_2),\ldots,(u_{2k+1},v_{2k+1})\}\subseteq X\times X$ of size $2k+1$ which is shattered by $\cH^{(s)}$. Let $V := \{u_1,u_2,\ldots, u_{2k+1}, v_1, v_2, \ldots, v_{2k+1}\}$ (note that in general we do not necessarily have that $|V| = 4k+2$ since the $u_i$ and $v_i$ need not all be distinct).

%We now show that $|V| \geq 2k + 1$ by considering the graph $G := (V, E)$, that is the graph with vertex set $V$ and edge set $E$. Note first that this graph cannot contain a cycle, since there is no labelling of $V$ which could induce a similarity labelling which labels all of the edges of a cycle except one with a $1$ and the remaining edge with a $0$. So the graph is a union of trees or `forest'. Note that in general the number of vertices in a forest is $|V| = |E| + r$, where $|E|$ is the number of edges and $r$ is the number of trees in the forest. In this case we have $|V| = 2k + r \geq 2k + 1$.

Let $G$ be the graph with vertex set $V$ and edge set $E$. Observe that elements of $\cH_k$ correspond to $\{0,1\}$-labellings of $V$ and that elements of $\cH^{(s)}_k$ correspond to $\{0,1\}$-labellings of $E$. Since $E$ is shattered by $\cH^{(s)}_k$, every labelling of $E$ is realisable as the induced map $h^{(s)}$ of some $h\in \cH_k$.

Note that $G$ cannot contain a cycle since there is no labelling of $V$ which could induce a similarity labelling on a cycle in which exactly one edge is labelled $0$ and the rest are labelled $1$\footnote{Indeed, under any such labelling of E any two vertices in the cycle are connected by two paths, one path containing exactly zero edges labelled with a $0$ (implying that the two vertices are labelled the same) and one path containing exactly one edge labelled with a $0$ (implying that the two vertices are labelled differently).}. So the graph is a union of trees, also known as a `forest'. Note that in general the number of vertices in a forest is $|V| = |E| + r$, where $|E|$ is the number of edges and $r$ is the number of trees in the forest. In this case we have $|V| = 2k + 1 + r$.

Now choose a labelling $g$,  which labels the vertices of each connected component (tree) in $G$ according to the following rule: for each connected component $C$ in $G$, label $\lfloor\frac{|C|}{2}\rfloor$ vertices $v\in C$ with a $1$ and the remaining $\lceil\frac{|C|}{2}\rceil$ with a $0$. Note that $g\notin \cH_k$ since

\begin{equation*}
\sum_{v\in V} g(v) = \sum_{C} \sum_{v\in C} g(v) = \sum_C{\left\lfloor\frac{|C|}{2}\right\rfloor} \geq \sum_C{\frac{|C| - 1}{2}} = \frac{|V| - r}{2} = k + \frac{1}{2} > k.
\end{equation*}
Consider the edge labelling $\pr{g^{(s)}}{E}$. Since $E$ is shattered by $\cH_k^{(s)}$, there must be some $h \in \cH_k$ such that $\pr{h^{(s)}}{E}=\pr{g^{(s)}}{E}$. But this is not possible, for if it were, then in order for $h^{(s)}$ to agree with $g^{(s)}$ we would need $\pr{h}{C} = \pr{g}{C}$ or $\pr{h}{C} = 1 - \pr{g}{C}$ for each connected component $C$ in $G$. Swapping the labellings between $0$ and $1$ on one or more of the connected components can only increase the number of $1$ labellings and thus 
\begin{equation*}
\sum_{v\in V} h(v) \geq \sum_{v\in V} g(v) > k
\end{equation*}
so $h$ cannot be in $\cH_k$. Thus we have found a labelling of $E$, namely $\pr{g^{(s)}}{E}$, which cannot be in $\cH_k^{(s)}$. But this is a contradiction of our initial assumption that $E$ was shattered by $\cH_k^{(s)}$. So we have proved that our assumption must have been incorrect and therefore $\vc{\cH_k^{(s)}} = 2k$.
\end{proof}

In Theorem~\ref{thm:vcdim}, the lower bound $\vc{\cH} -1  \le \vc{\cH^{(s)})}$ is tight, for example when $\cH = \{0,1\}^{[n]}$.  
However, observe that in Example~\ref{ex:sparsek}, the hypothesis space of $k$-sparse vectors, the similarity space ``expands'' only by a factor of 2, which is less than the factor $\delta=4.55$ of Theorem~\ref{thm:vcdim}.
We leave as a conjecture that the upper bound in Theorem~\ref{thm:vcdim} can be improved to a factor of two.

\vspace{.1in}
\noindent
{\bf Acknowledgements.}  We would like to thank Shai Ben-David, Ruth Urner and Fabio Vitale for valuable discussions.  In particular we would like thank Ruth Urner for proving an initial motivating upper bound of  $\vc{\cH^{(s)}} \le 2\vc{\cH} \log(2\vc{\cH})$. 

\bibliography{vcsimbib-1}

%s\bibliographystyle{abbrv}
\bibliographystyle{alpha}

\end{document}